\documentclass[conference]{IEEEtran}
\IEEEoverridecommandlockouts
\usepackage{amsmath,amsfonts}
\usepackage{array}
\usepackage{textcomp}
\usepackage{stfloats}
\usepackage{url}
\usepackage{amsmath, amssymb, amsthm}
\usepackage{verbatim}
\usepackage{graphicx}
\usepackage{amssymb}
\usepackage{hyperref}
\usepackage{graphicx}
\usepackage{amsmath}
\usepackage{longtable}
\usepackage{algorithm} 
\usepackage{algpseudocode} 
\usepackage{mathrsfs}
\usepackage{subcaption}
\usepackage{mathtools}
\usepackage{pifont}
\usepackage{color}
\usepackage{lineno}
\usepackage{graphicx}  
\usepackage{makecell} 
\usepackage{pdflscape}
\usepackage{adjustbox}
\usepackage[utf8]{inputenc}
\usepackage{tabularx}
\usepackage{blindtext}
\usepackage{longtable}
\usepackage{lscape}
\usepackage{amsthm}
\usepackage{graphicx}
\usepackage{subcaption} 
\usepackage{caption}

\usepackage{setspace}
\usepackage{notoccite} 
\usepackage{lscape} 
\usepackage{mwe}
\usepackage{booktabs}
\usepackage{amsthm}
\newtheorem{theorem}{Theorem}[section]

\theoremstyle{definition}

\theoremstyle{definition}

\newcommand{\RNum}[1]{\lowercase\expandafter{\romannumeral #1\relax}}
\newcommand{\RNumU}[1]{\uppercase\expandafter{\romannumeral #1\relax}}
\usepackage[numbers]{natbib}

\def\BibTeX{{\rm B\kern-.05em{\sc i\kern-.025em b}\kern-.08em
    T\kern-.1667em\lower.7ex\hbox{E}\kern-.125emX}}
\begin{document}

\title{CI-RKM: A Class-Informed Approach to Robust Restricted Kernel Machines
}

\author{
\IEEEauthorblockN{Ritik Mishra}
\IEEEauthorblockA{
\textit{Department of Mathematics} \\
\textit{Indian Institute of Technology Indore}\\
phd2301241003@iiti.ac.in}
\and
\IEEEauthorblockN{Mushir Akhtar}
\IEEEauthorblockA{
\textit{Department of Mathematics} \\
\textit{Indian Institute of Technology Indore}\\
phd2101241004@iiti.ac.in}
\and
\IEEEauthorblockN{M. Tanveer\textsuperscript{*}\thanks{\textsuperscript{*}Corresponding author}}
\IEEEauthorblockA{
\textit{Department of Mathematics} \\
\textit{Indian Institute of Technology Indore}\\
mtanveer@iiti.ac.in}
}

\maketitle
\begin{abstract}
Restricted kernel machines (RKMs) represent a versatile and powerful framework within the kernel machine family, leveraging conjugate feature duality to address a wide range of machine learning tasks, including classification, regression, and feature learning. However, their performance can degrade significantly in the presence of noise and outliers, which compromises robustness and predictive accuracy. In this paper, we propose a novel enhancement to the RKM framework by integrating a class-informed weighted function. This weighting mechanism dynamically adjusts the contribution of individual training points based on their proximity to class centers and class-specific characteristics, thereby mitigating the adverse effects of noisy and outlier data. By incorporating weighted conjugate feature duality and leveraging the Schur complement theorem, we introduce the class-informed restricted kernel machine (CI-RKM), a robust extension of the RKM designed to improve generalization and resilience to data imperfections. Experimental evaluations on benchmark datasets demonstrate that the proposed CI-RKM consistently outperforms existing baselines, achieving superior classification accuracy and enhanced robustness against noise and outliers. Our proposed method establishes a significant advancement in the development of kernel-based learning models, addressing a core challenge in the field.
\end{abstract}

\begin{IEEEkeywords}
Restricted kernel machines (RKM), Class-informed weights, Kernel Methods, Weighted conjugate feature duality.
\end{IEEEkeywords}

\section{Introduction}
\IEEEPARstart{K}{ernel} methods have long been at the forefront of machine learning, offering robust tools for addressing complex datasets characterized by non-linear relationships \cite{wang2022support,wilson2016deep}. These methods operate by transforming data from its original input space into a higher-dimensional reproducing kernel Hilbert space (RKHS) \cite{mercer1909xvi} using a feature map \( \phi(\cdot): X \to H \). This transformation enables the resolution of intricate non-linear problems in the input space through linear operations within the RKHS. When the optimization problem is formulated in terms of inner products between data points, kernel functions \( k(x, y) = \langle \phi(x), \phi(y) \rangle_H \) can be utilized, a concept known as the kernel trick \cite{scholkopf2000kernel}. The kernel trick is particularly advantageous in scenarios involving high-dimensional or even infinite-dimensional feature mappings, as it circumvents the explicit computation of the feature map while retaining computational efficiency \cite{hofmann2006review}. This makes kernel methods highly effective for a range of machine learning tasks, including classification \cite{elisseeff2001kernel}, \cite{kumari2024diagnosis}, regression \cite{soman2009machine}, feature extraction \cite{kocsor2004kernel}, and dimensionality reduction \cite{wu2019solving}. Prominent methods that utilize the kernel trick to effectively deal with complex, non-linear data structures include support vector machines (SVMs) \cite{vapnik2013nature}, least squares support vector machines (LS-SVMs) \cite{suykens1999least}, and kernel principal component analysis (kernel PCA) \cite{scholkopf1997kernel}.


Among the innovative models within the kernel machine family, the restricted kernel machine (RKM) \cite{10.1162/neco_a_00984} has emerged as a promising approach to kernel-based learning. The RKM model, as proposed by \citet{10.1162/neco_a_00984} \cite{10.1162/neco_a_00984}, extends the LS-SVM framework by incorporating principles from restricted Boltzmann machines (RBMs) \cite{hinton2005kind}. This integration leverages the concept of conjugate feature duality, wherein the primal and dual variables correspond to visible and hidden layers, respectively. This formulation enables the RKM to transform data into a high-dimensional feature space, facilitating the construction of linear separating hyperplanes, akin to traditional kernel methods \cite{houthuys2021tensor}.

One of the RKM's most compelling features is its capacity to form deeper architectures by stacking multiple RKMs, allowing for hierarchical and complex feature learning. This capability significantly enhances its flexibility and performance, particularly when dealing with high-dimensional data \cite{10.1162/neco_a_00984}. Over the years, RKMs have been further developed and refined, demonstrating their effectiveness across a wide range of tasks, including supervised and unsupervised learning \cite{pandey2021generative}, \cite{HOUTHUYS202154}.  Additionally, the RKM framework bridges the gap between traditional kernel methods and deep learning paradigms, offering a versatile and robust alternative to conventional kernel-based models. This combination of theoretical rigor and practical adaptability positions the RKM as a valuable tool in the broader landscape of machine learning.

The RKM framework has undergone various adaptations and extensions to tackle a diverse range of machine learning challenges. For example, RKMs have been successfully employed in multi-view classification tasks, where the model efficiently integrates and processes multiple data modalities through tensor-based representations \cite{houthuys2021tensor}. Furthermore, RKMs have demonstrated remarkable effectiveness in learning disentangled representations, which are particularly valuable for unsupervised learning tasks such as outlier detection and anomaly identification \cite{tonin2021unsupervised}. These disentangled features offer a clearer understanding of the underlying structure of the data, enabling the model to distinguish between meaningful variations and noise. In addition to their application in classification and representation learning, RKMs have been utilized in generative modeling tasks. By learning a latent space that encapsulates the data distribution, RKMs can generate new data points through sampling, making them a powerful tool for data generation and predictive modeling \cite{pandey2021generative}, \cite{schreurs2018generative}. This versatility underscores the RKM’s potential to address a broad spectrum of machine learning problems. However, despite these strengths, RKMs face a significant limitation. Their reliance on conjugate features for classification does not account for the geometric relationships between data points, rendering the model susceptible to the adverse effects of noise and outliers. While RKMs perform exceptionally well on clean and well-structured datasets, their robustness diminishes in the presence of noisy or incomplete information. Addressing this limitation is essential for further enhancing the model's applicability and reliability in real-world scenarios.

Robust loss functions and weighting schemes are fundamental components in enhancing the robustness of machine learning models, particularly in handling noise and outliers. Recently various robust loss functions, such as wave loss \cite{akhtar2024advancing}, RoBoSS loss \cite{10685140}, guardian loss \cite{akhtar2024gl}, and others \cite{akhtar2024hawkeye}, have been introduced to improve resilience against anomalous data. In parallel, weighting schemes have also been extensively explored to mitigate the adverse effects of noise and outliers  \cite{lin2002fuzzy}, \cite{kumari2024class}, \cite{TANVEER2024127712}. These weighting functions assign adaptive weights to data points based on
their distance from the class center, effectively reducing the
influence of outliers and noisy samples on the learning process
\cite{10849606}, \cite{akhtar2024flexi}. By prioritizing more representative samples while
diminishing the impact of anomalous data, weighted functions
provide an elegant mechanism to improve the generalization
performance of machine learning models in diverse scenarios.

Building on the remarkable success of weighting strategies in handling noise and outliers, in this work, we introduce a novel extension to the RKM framework that integrates geometric and class-specific information into the learning process. Our approach seamlessly incorporates a class-informed weighted function into RKM architecture, which dynamically adjusts the influence of each data point based on its class membership and proximity to the class center. By explicitly accounting for the geometric relationships between data points, this weighting mechanism enhances the model’s robustness against noise and outliers, leading to more reliable and accurate predictions. The proposed method, termed the class-informed restricted kernel machine (CI-RKM), seamlessly integrates the class-informed weight function into the RKM framework. This enhancement equips the model with the ability to better generalize to real-world datasets, even when faced with imperfections such as noisy or incomplete data. Extensive experimental evaluations validate the effectiveness of CI-RKM, demonstrating its superior classification accuracy and robustness compared to baseline models. The proposed approach represents a significant advancement in the kernel machine family, offering a versatile and resilient solution for a wide range of machine learning applications.

The key contributions of this work are summarized as follows:
\begin{itemize}
    \item We propose the class-informed restricted kernel machine (CI-RKM), a robust extension of RKM, which incorporates a class-informed weighted function to improve resilience against noise and outliers in classification tasks.
    \item We leverage weighted conjugate feature duality and the Schur complement theorem in the development of CI-RKM, enabling enhanced resilience to data imperfections and outperforming baseline models in terms of both robustness and generalization.
    \item We validate the robustness and generalization capabilities of CI-RKM through extensive experiments on benchmark datasets, demonstrating its superior performance in terms of classification accuracy and resilience compared to traditional RKMs and other baseline models.
\end{itemize}
The remainder of the paper is organized as follows: Section \ref{Related-Work} provides an overview of the RKM framework. Section \ref{Proposed-work} introduces the proposed class-informed restricted kernel machine (CI-RKM), detailing the class-informed weighted function and its integration into the RKM framework. Section \ref{Experiment-section} presents the experimental setup, results, and a comprehensive evaluation of CI-RKM's performance against baseline models. Finally, Section \ref{Conclusions-section} concludes the paper with potential future research directions.

\section{Related Work}\label{Related-Work}
In this section, we discuss the mathematical formulation of the restricted kernel machine for binary classification problems.
\subsection{Restricted Kernel Machine (RKM)}

The RKM, proposed by Suykens et al. \cite{10.1162/neco_a_00984}, builds upon the Least Squares Support Vector Machine framework, taking inspiration from RBM \cite{hinton2005kind}. It introduces an approach by linking primal and dual variables through visible and hidden layers, analogous to the RBM architecture. This dual formulation is achieved using conjugate feature duality, allowing the RKM to model complex relationships effectively.


\par
Similar to LS-SVM, the RKM method utilizes the kernel trick to map the data into a higher-dimensional feature space, where a linear decision boundary can be constructed, see fig. \ref{fig:side_by_side}. The optimization problem for classification tasks is formulated as follows.

Consider a training set consisting of \( N \) data points \( \{ (\mathbf{x}_k, y_k) \}_{k=1}^N \), where \( \mathbf{x}_k \in \mathbb{R}^d \) represents the \( k \)-th input vector and \( y_k \in \{-1, 1\} \) denotes the corresponding class label. The objective function for the RKM classification model is defined as:
\begin{equation}
\mathcal{J} = \frac{\eta}{2} w^T w + \sum_{k=1}^N \left( 1 - \left(  w^T \sigma(x_k) + b \right) y_k \right) h_k - \frac{\lambda}{2} \sum_{k=1}^N h_k^2.
\label{eqrkm}
\end{equation}
Here, \( \mathbf{w} \) represents the weight vector, \( b \) is the bias term, and \( \lambda \) and \( \eta \) are regularization coefficients. The hidden feature values are denoted by \( h_k \in \mathbb{R} \). The feature map \( \sigma: \mathbb{R}^d \to \mathbb{R}^{d_h} \) is responsible for mapping the input vectors to a higher-dimensional space, where \( d \) is the dimensionality of the input space, and \( d_h \) is the dimension of the feature space into which the data are projected. Typically, this feature map is implicitly defined through a positive-definite kernel \( K(\mathbf{x}_i, \mathbf{x}_j) \), which specifies the inner product between the mapped feature vectors:
\begin{equation}
K(\mathbf{x}_i, \mathbf{x}_j) = \sigma(\mathbf{x}_i)^T \sigma(\mathbf{x}_j).
\end{equation}
By using this kernel trick, RKM avoids explicitly defining the high-dimensional feature space, allowing it to operate in spaces of high or even infinite dimensionality.

By solving the optimization problem and finding the stationary points of \(\mathcal{J}\), the RKM reduces to a linear system that can be efficiently solved. Additionally, RKM can be extended to multiclass classification by using encoding strategies such as one-vs-all (OVA) or minimum output encoding (MOC), where separate binary classifiers are trained for each class.

In the binary classification setting, as derived from equation (\ref{eqrkm}), the RKM can be formulated as a system of linear equations:
\begin{equation}
\left[\begin{array}{c|c}
\frac{1}{\eta} M + \lambda I_N & 1_N \\
\hline
1_N^T & 0
\end{array}\right]
\left[\begin{array}{c}
y \odot h \\
\hline
b
\end{array}\right]
=
\left[\begin{array}{c}
y \\
\hline
0
\end{array}\right].
\end{equation}
Here, \({M}\) represents the kernel matrix for all training data points, \({y}\) are the class labels, and \(\odot\) denotes the element-wise product. The kernel matrix \({M}\) is computed as:
\begin{equation}
M_{ij} = \sigma(\mathbf{x}_i)^T \sigma(\mathbf{x}_j) = K(\mathbf{x}_i, \mathbf{x}_j), \quad i,j = 1, \dots, N.
\end{equation}

This approach effectively handles binary classification problems by training a classifier and leveraging the kernel function to compute the necessary inner products in the transformed feature space.


\begin{figure*}
\centering
\subcaptionbox{} { %
\includegraphics[width=0.40\textwidth,keepaspectratio]{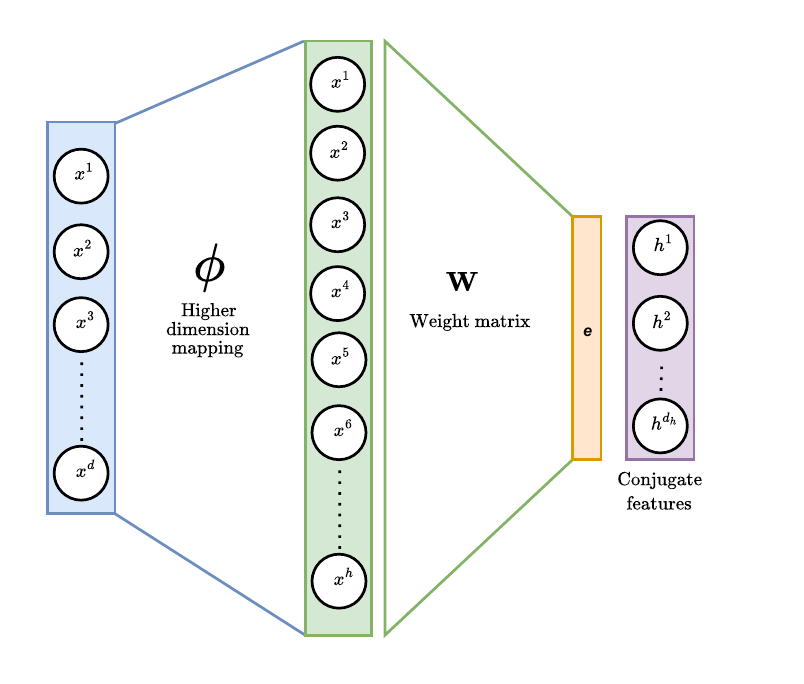}}
      \hspace{2mm}
      \subcaptionbox{} { %
\includegraphics[width=0.40\textwidth,keepaspectratio]{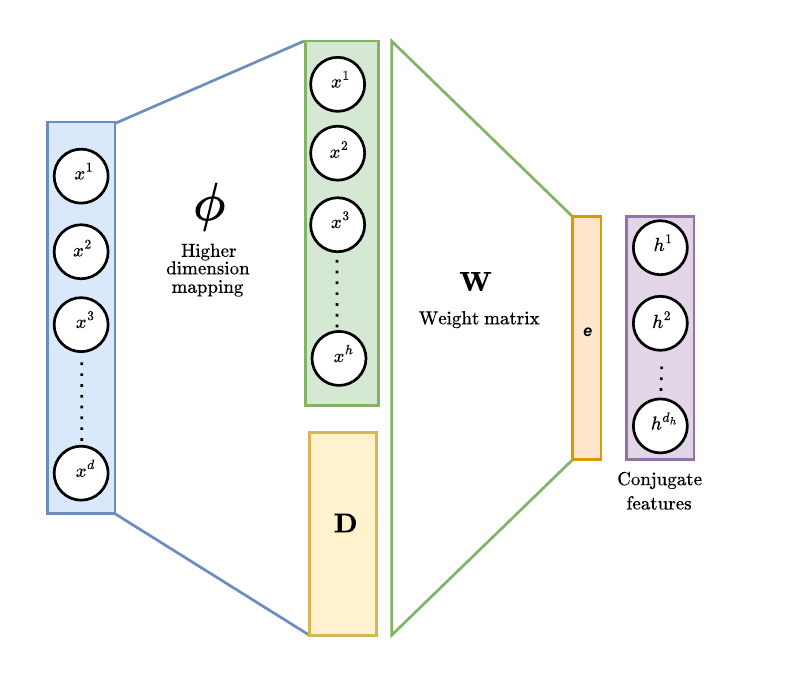}}
      \caption{Subfigure (a) illustrates the traditional RKM, which uses a standard kernel function for mapping input data into a higher-dimensional space. In contrast, subfigure (b) presents the proposed CI-RKM, where the conjugate features are modified by incorporating class-specific information, enhancing the model's ability to capture class distinctions in the data.}
    \label{fig:side_by_side}
 \end{figure*}

\section{Proposed Work}\label{Proposed-work}

In this section, we first discuss the class-informed weight function and then proceed to present the mathematical formulation of the proposed class-informed restricted kernel machine (CI-RKM) for classification problems.

\subsection{Class-Informed Weighted Function}

Inspired by the weighted schemes used in various machine learning models \cite{ha2013support}, we aim to incorporate a weight that captures not only the geometric characteristics of the data but also provides robustness to noise and outliers. The class-informed weight function  \( D \) assesses the proximity of each sample to the centroid of its corresponding class within a high-dimensional feature space. For a given training sample \( x_k \), the weight function is defined as:

\begin{equation}
D(x_k) = 
\begin{cases} 
1 - \frac{\| \sigma(x_k) - K^+ \|}{r^+ + \xi} , & y_k = +1, \\
1 - \frac{\| \sigma(x_k) - K^- \|}{r^- + \xi} , & y_k = -1,
\end{cases}
\label{eqweight}
\end{equation}

where $\xi>0$ is small positive number, \( K^+ \) and \( K^- \) denote the centroids of the positive and negative classes, respectively, \( r^+ \) and \( r^- \) represent the radii of these classes. The distance \( r \) between two samples is defined as:

\begin{equation}
r(\sigma(x_i), \sigma(x_j)) = \| \sigma(x_i) - \sigma(x_j) \|,
\end{equation}
where \( \| \cdot \| \) denotes the Frobenius norm. The centroid for each class is calculated as:

\begin{equation}
\mathcal{K^+} = \frac{1}{N^+} \sum_{y_k = 1} \sigma(x_k) \quad \text{and} \quad \mathcal{K^-} = \frac{1}{N^-} \sum_{y_k = -1} \sigma(x_k),
\end{equation}

where \( N^+ \) and \( N^- \) represent the number of samples in the positive and negative classes, respectively. The radius of each class is defined as:

\begin{equation}
r^\pm = \max_{y_i = \pm 1} \| \sigma(x_i) - \sigma(x_j) \|.
\end{equation}

\begin{theorem}
    \cite{ha2013support}  Let \( K(x, x') \) be a kernel function. Then the inner product distance is given by:

    \begin{equation}
\left\|\sigma(x) - \sigma(x')\right\| = \sqrt{K(x, x) + K(x, x') - 2K(x, x')}.
\end{equation}
\end{theorem}
\begin{proof}
    
\begin{equation*}
\begin{aligned}
&\left\|\sigma(x) - \sigma(x')\right\| = \sqrt{\left(\sigma(x) - \sigma(x')\right) \cdot \left(\sigma(x) - \sigma(x')\right)} \\
&= \sqrt{(\sigma(x) \cdot \sigma(x)) + (\sigma(x') \cdot \sigma(x')) - 2(\sigma(x) \cdot \sigma(x'))} \\
&= \sqrt{K(x, x) + K(x, x') - 2K(x, x')}.
\end{aligned}
\end{equation*}
\end{proof}
\subsection{Class Informed Restricted Kernel Machine (CI-RKM)}
We generalize the concept of conjugate feature duality by introducing a weighting matrix \( D \). Let \( D \succ 0 \) represent a positive-definite diagonal matrix. For two vectors \( e, h \in \mathbb{R}^n \) and a scalar \( \lambda > 0 \), the following inequality holds:
\begin{equation}
\frac{1}{2 \lambda} e^{\top} D e + \frac{\lambda}{2} h^{\top} D^{-1} h \geq e^{\top} h.
\label{eqweightcong}
\end{equation}

This inequality can be validated using the Schur complement, which reformulates the inequality in quadratic form:
\begin{equation}
\frac{1}{2} \left[ \begin{array}{ll}
e^{\top} & h^{\top}
\end{array} \right] 
\left[ \begin{array}{cc}
\frac{1}{\lambda} D I & -I \\
-I & \lambda D^{-1} I
\end{array} \right] 
\left[ \begin{array}{l}
e \\
h
\end{array} \right] \geq 0.
\end{equation}
From this, we deduce that for a matrix \( Q = \left[ \begin{array}{cc} A & B \\ B^{\top} & C \end{array} \right] \), the condition \( Q \succeq 0 \) holds if and only if \( A \succ 0 \) and the Schur complement \( C - B^{\top} A^{-1} B \succeq 0 \). This formulation leads to the Fenchel-Young inequality for quadratic functions.

 The constraints for the LS-SVM classifier are given by:

\begin{equation}
y_k \left( w^T \sigma(x_k) + b \right) = 1_p - e_k, \quad k = 1, \dots, N,
\end{equation}

where \( y_k \in \{-1, 1\}^p \), \( e_k \in \mathbb{R}^p \) represents the class label encoding, and \( Y_k = \operatorname{diag} \left\{ y_{(1), k}, \dots, y_{(p), k} \right\} \) is a diagonal matrix.

Starting from the objective function of the LS-SVM with weight \( D \) incorporated into the restricted kernel machine see fig. \ref{fig:side_by_side}, we arrive at the following formulation:

\begin{equation}
\begin{aligned}
\mathcal{J} = & \frac{\eta}{2} \operatorname{Tr}(w^T w) + \frac{1}{2 \lambda} \sum_{k=1}^N e_k^T D e_k, \\
& \quad \text{subject to} \quad e_k = 1_p - Y_k \left( w^T \sigma(x_k) + b \right), \quad \forall k.
\end{aligned}
\end{equation}
Where D is the class informed weight calculated by equation (\ref{eqweight}).\\
Now applying weight conjugate feature duality by equation (\ref{eqweightcong}), we get:
\begin{equation}
\begin{aligned}
\mathcal{J} \geq & \sum_{k=1}^N e_k^T h_k - \frac{\lambda}{2} \sum_{k=1}^N h_k^T D^{-1} h_k + \frac{\eta}{2} \operatorname{Tr}(w^T w), \\
& \quad \text{subject to} \quad e_k = 1_p - Y_k \left( w^T \sigma(x_k) + b \right), \quad \forall k \\
= & \sum_{k=1}^N \left( 1_p^T - \left( \sigma(x_k)^T w + b^T \right) Y_k \right) h_k \\
&- \frac{\lambda}{2} \sum_{k=1}^N h_k^T D^{-1} h_k + \frac{\eta}{2} \operatorname{Tr}(w^T w) \approx \underline{\mathcal{J}}.
\end{aligned}
\end{equation}

The stationary points of \( \underline{ \mathcal{J}}(h_k, w, b) \) are found by solving the system of equations given below:

\begin{equation}
\left\{
\begin{array}{l}
\frac{\partial \underline{\mathcal{J}}}{\partial h_k} = 0 \Rightarrow 1_p = Y_k \left( w^T \sigma(x_k) + b \right) + \lambda D^{-1} h_k, \quad \forall k \\
\frac{\partial \underline{\mathcal{J}}}{\partial w} = 0 \Rightarrow w = \frac{1}{\eta} \sum_k \sigma(x_k) h_k^T Y_k, \\
\frac{\partial \underline{\mathcal{J}}}{\partial b} = 0 \Rightarrow \sum_k Y_k h_k = 0.
\end{array}
\right.
\end{equation}

The solution in conjugate features follows from the linear system:

\begin{equation}
\left[ \begin{array}{c|c}
\frac{1}{\eta} M + \lambda D^{-1} I_N & 1_N \\
\hline 1_N^T & 0
\end{array} \right]
\left[ \begin{array}{c}
H_y^T \\
\hline
b^T
\end{array} \right]
= \left[ \begin{array}{c}
Y^T \\
\hline
0
\end{array} \right],
\end{equation}
where \( H_y = \left[ y_1 h_1, \dots, y_N h_N \right] \).

The primal and dual model representations are expressed in terms of the feature map and kernel function. The dual representation is given by:

\begin{equation}
(\mathcal{D})_{\text{RKM}}: \hat{y} = \operatorname{sign}\left[ \frac{1}{\eta} \sum_j  h_{y_j} M(x_j, x) + b \right],
\end{equation}
where \( M(x_j, x) \) is the kernel function between the new data point $x_j$ and input data points.

\begin{theorem}
The class informed weight matrix \( diag(D) \) is positive definite, i.e., for all non-zero vector \( v \in \mathbb{R}^n \), the quadratic form \( v^\top D v > 0 \).
\end{theorem}

\begin{proof}
 \( D(x_k) \) represents a class-informed weight associated with the sample \( x_k \), which depends on the distance between \( \sigma(x_k) \) (the feature representation of \( x_k \)) and the class centroids \( K^+ \) or \( K^- \). Since the distance function in equation \ref{eqweight} is positive for all $x_k$, we have \( D(x_k) > 0 \) for all \( k \). Hence, all diagonal entries of \( \text{diag}(D) \) are positive.

Next, consider the quadratic form of \( \text{diag}(D) \) for any vector \( v \in \mathbb{R}^n \):
\[
v^\top \text{diag}(D) v = \sum_{k=1}^n v_k^2 D(x_k).
\]
Since \( D(x_k) > 0 \) and \( v_k^2 > 0 \) for all \( k \), each term \( v_k^2 D(x_k) \) is positive. Therefore, the entire sum is positive.
\[
v^\top \text{diag}(D) v > 0.
\]
Thus, \( \text{diag}(D) \) is positive semidefinite.
\end{proof}

\section{Experiments, Results and Discussion}\label{Experiment-section}



In this section, we first provide an overview of the compared models, the experimental setup, and the hyperparameter settings. Next, we conduct an in-depth analysis of the results, including a detailed discussion of the statistical tests performed. Additionally, we present an ablation study to evaluate the impact of the proposed weighting scheme on the performance of the CI-RKM model.

\subsection{Models Compared}
To assess the performance of the proposed CI-RKM model, we compare it against several established classification models. The following baseline models are included in the comparison:

\begin{itemize}
\item RVFL: Random vector functional link (RVFL) model, as described by \cite{pao1994learning}.

\item RVFLwoDL or ELM: RVFL without direct link or extreme learning machine, as introduced by \cite{huang2006extreme}.

\item IF-RVFL: Intuitionistic fuzzy RVFL model, proposed by \cite{malik2022alzheimer}.

\item NF-RVFL: A Neuro-fuzzy RVFL model, as discussed in \cite{10416391}.

\item RKM: Restricted kernel machine (RKM), as outlined in \cite{10.1162/neco_a_00984}.

\item CI-RKM: The proposed class-informed restricted kernel machine (CI-RKM) model, which is the focus of this study.
\end{itemize}

\subsection{Experimental Setup}
The experiments are conducted on a computing system running Python, featuring an Intel(R) Xeon(R) Platinum 8260 CPU with 24 cores and 48 logical processors, clocked at 2.30 GHz, 256 GB of RAM, and operating on the Windows 10 platform. A 5-fold cross-validation method is utilized, along with grid search for hyperparameter tuning. The testing accuracy for each fold was computed independently for each set of hyperparameters. The average testing accuracy for each set of hyperparameters was then calculated by taking the mean of the accuracies from the five folds. 

The regularization parameters for all the models are chosen from the set \(\{10^{-5}, 10^{-4}, \dots, 10^5\} \). Both the proposed CI-RKM and RKM models employ an RBF kernel, with parameters selected from the range \( \{2^{-5}, 2^{-4}, \dots, 2^5\} \). For RVFL, IF-RVFL, and NF-RVFL models, the number of nodes in the hidden layer is selected from the range \( \{3:20:203\} \) . In NF-RVFL model, the number of fuzzy rules in the fuzzy layer is chosen from the range \(\{5:5:50\}\), with a standard deviation of 1.

\subsection{Performance Evaluation and Statistical Analysis on UCI Datasets}

Table \ref{tab:table1} presents the experimental results for the proposed CI-RKM model on 26 benchmark UCI datasets \cite{dua2017uci}, compared against various baseline models: RKM, RVFLwoDL \cite{huang2006extreme}, RVFL \cite{pao1994learning}, IF-RVFL \cite{malik2022alzheimer}, and NF-RVFL \cite{10416391}. Among these, CI-RKM achieves the highest average accuracy, followed closely by RKM and NF-RVFL, clearly outperforming the remaining baseline models. The proposed CI-RKM achieves average accuracy of 85.94\% and other baseline models as RKM, RVFLwoDL \cite{huang2006extreme}, RVFL \cite{pao1994learning}, IF-RVFL \cite{malik2022alzheimer}, and NF-RVFL \cite{10416391} attains average accuracy scores 84.87\%, 77.57\%, 77.80\%, 77.73\%, and 79.66\%, respectively. 

While the average accuracy provides a general view of model performance, it may not fully capture performance variations across different datasets. To gain a deeper understanding, we employed a series of statistical tests \cite{demvsar2006statistical}, including the ranking test, the Friedman test, and the Nemenyi post hoc test, which offer a more comprehensive and unbiased comparison of the models. In the ranking method, each model is assigned a rank based on its performance on each dataset. The rank of the $m$-th model on the $d$-th dataset is represented as $p (m, d)$, and the average rank across all datasets is computed as:
\[
p(m, \ast) = \frac{1}{D} \sum_{d=1}^D p(m, d).
\]
The average ranks for all models are summarized in the last row of Table \ref{tab:table1}. CI-RKM, with an average rank of 1.81, has the lowest rank, followed by RKM and NF-RVFL, indicating overall best performance of the proposed CI-RKM. To assess whether there are significant performance differences among the models, we applied the Friedman test to the average ranks. The Friedman test follows a chi-squared distribution ($\chi^2_F$) with $M-1$ degrees of freedom, where $M$ is the number of models and $D$ is the number of datasets. The Friedman test statistic is calculated as:
\[
\chi^2_F = \frac{12D}{M(M+1)} \left( \sum_{m=1}^M \left( p(m, \ast) \right)^2 - \frac{M(M+1)^2}{4} \right).
\]
Table \ref{tab:friedmann} shows the results of the Friedman test. We computed the $\chi^2_F$ value as 2.86, with degrees of freedom $F(2,125)$, and the critical value for the F-distribution as 2.15. Since the test statistic exceeds the critical value, we reject the null hypothesis, confirming that the models show statistically significant differences in performance. Further, to examine pairwise differences, the Nemenyi post hoc test is conducted. This test checks whether the average rank of one model is significantly lower than another by at least the critical difference ($C.D.$). The critical difference is computed as:

\[
C.D. = q_{\alpha} \cdot \sqrt{\frac{M(M+1)}{6D}}
\]

where $q_{\alpha}$ is the critical value for the two-tailed Nemenyi test. At a significance level of $\alpha = 0.05$, the critical difference is found to be 1.47. Detailed results of the Nemenyi post hoc test can be seen in Table \ref{tab:posthoc}. These results confirm that CI-RKM outperforms all other models, with RKM also outperforming RVFLwoDL, RVFL, and IF-RVFL.

In summary, the CI-RKM model shows superior performance in terms of both accuracy and stability. Its higher average accuracy, combined with the statistical test results, underscores its significant advantage in terms of generalization compared to the baseline models.

\begin{table*}[]
\centering
\caption{Classification Accuracies of the RKM, RVFLwoDL, RVFL, IF-RVFL, NF-RVFL, and the proposed CI-RKM. The last row of this table represents the average rank of each model.}
\label{tab:table1}
\resizebox{16cm}{!}{%
\begin{tabular}{lcccccc}
\hline
\textbf{Dataset} $\downarrow$ $|$ \textbf{Model} $\rightarrow$  &
  {\textbf{RKM} \cite{10.1162/neco_a_00984}} &
  {\textbf{RVFLwoDL} \cite{huang2006extreme}} &
  {\textbf{RVFL} \cite{pao1994learning}} &
  {\textbf{IF-RVFL} \cite{malik2022alzheimer}} &
  {\textbf{NF-RVFL} \cite{10416391}} &
  {\textbf{CI-RKM}}$^{\dagger}$ \\ \hline
  
bank                              & 92.04 & 89.49 & 89.41 & 89.12 & 89.91 & 91.05          \\
 
blood                             & 78.67 & 76.91 & 76.51 & 77.44 & 77.44 & 78             \\
 
breast\_cancer\_wisc              & 97.14 & 87.99 & 88.57 & 89.85 & 87.71 & 97.14 \\
 
breast\_cancer\_wisc\_prog        & 70.69 & 70.18 & 70.18 & 71.93 & 70.9  & 80    \\
 
chess\_krvkp                      & 99.38 & 71.94 & 72.03 & 72.63 & 85.23 & 99.38 \\
 
congressional\_voting             & 63.22 & 63.22 & 63.68 & 58.85 & 64.14 & 63.22          \\
 
conn\_bench\_sonar\_mines\_rocks  & 92.86 & 60.52 & 62.08 & 54.83 & 63.46 & 95.24 \\
 
cylinder\_bands                   & 81.55 & 65.81 & 66.42 & 63.49 & 70.13 & 80.58          \\
 
echocardiogram                    & 81.48 & 83.9  & 84.67 & 80.77 & 85.44 & 81.48          \\
 
fertility                         & 95    & 91    & 91    & 92    & 91    & 95    \\
 
haberman\_survival                & 79.03 & 73.82 & 73.49 & 75.13 & 75.46 & 77.42          \\
 
heart\_hungarian                  & 77.97 & 74.49 & 73.82 & 76.88 & 78.25 & 72.88          \\
 
mammographic                      & 81.35 & 79.3  & 79.92 & 79.71 & 79.19 & 79.79          \\
 
monks\_2                          & 83.47 & 81.68 & 80.01 & 82.67 & 87.38 & 82.64          \\
 
monks\_3                          & 98.2  & 90.79 & 91.16 & 90.07 & 92.17 & 97.3           \\
 
musk\_1                           & 94.79 & 69.77 & 72.06 & 71.86 & 74.16 & 92.71          \\
 
oocytes\_trisopterus\_nucleus\_2f & 89.07 & 77.52 & 78.94 & 75.22 & 79.93 & 86.89          \\
 
parkinsons                        & 82.05 & 80.51 & 80.51 & 78.46 & 84.1  & 82.05          \\
 
pima                              & 78.57 & 72.27 & 72.01 & 73.83 & 72.66 & 77.27          \\
 
pittsburg\_bridges\_T\_OR\_D      & 80.95 & 87.19 & 87.19 & 89.19 & 90.19 & 85.71          \\
 
planning                          & 75.68 & 71.38 & 71.38 & 69.8  & 71.94 & 75.68          \\
 
spectf                            & 87.04 & 79.72 & 79.34 & 79.34 & 79.72 & 85.19          \\
 
statlog\_heart                    & 90.74 & 80    & 80.37 & 81.85 & 81.85 & 88.89          \\
 
tic\_tac\_toe                     & 100   & 88.93 & 88.83 & 81.43 & 79.32 & 100   \\
 
titanic                           & 80.73 & 77.92 & 77.92 & 79.05 & 79.05 & 79.82          \\
 
vertebral\_column\_2clases        & 90.32 & 70.65 & 71.29 & 85.48 & 80.32 & 93.55 \\ \hline
 
\textbf{Average Accuracy}                  & 84.87 & 77.57 & 77.8  & 77.73 & 79.66 & 85.94 \\ \hline
\textbf{Average Rank} &
  2.15 &
  4.58 &
  4.35 &
  4.31 &
  3 &
  1.81  \\ \hline
  \multicolumn{7}{l}{\(^{\dagger}\) denotes the proposed model.}\\
\end{tabular}%
}
\end{table*}

\begin{table}[h]
\centering
\caption{Results of the Friedman test on the UCI dataset.}
\label{tab:friedmann}
\resizebox{9cm}{!}{%
\begin{tabular}{|l|c|c|c|c|c|c|}
\hline
\textbf{Dataset} & $M$ & $D$ & $\chi_f^2$ & $F_F$ & $F_{(M-1), (M-1)(D-1)}$ & \begin{tabular}[c]{@{}c@{}}\textbf{Significant difference}\\ (\textbf{As per Friedman test})\end{tabular} \\
\hline
\textbf{UCI Dataset} & 6 & 26 & 13.37 & 2.86 & 2.15 & Yes \\ 
\hline
\end{tabular}%
}
\end{table}

\begin{table}[h]
\centering
\caption{Differences in the rankings of the proposed CI-RKM model against baseline models.}
\label{tab:posthoc}
\resizebox{9cm}{!}{%
\begin{tabular}{|l|c|c|c|}
\hline
\textbf{Model} & \textbf{Average Rank} & \textbf{Rank Difference} & \begin{tabular}[c]{@{}c@{}}\textbf{Significant difference}\\ (\textbf{As per Nemenyi post-hoc test})\end{tabular}\\
\hline
\textbf{RVFLwoDL} \cite{huang2006extreme} & 4.58 & 2.77 & Yes \\
\textbf{RVFL} \cite{pao1994learning} & 4.35 & 2.54 & Yes \\
\textbf{IF-RVFL} \cite{malik2022alzheimer} & 4.31 & 2.5 & Yes \\
\textbf{NF-RVFL} \cite{10416391} & 3 & 1.19 & No \\
\textbf{RKM} \cite{10.1162/neco_a_00984} & 2.15 & 0.34 & No \\
\textbf{CI-RKM}$^{\dagger}$ & 1.81 & - & N/A \\ \hline
\multicolumn{4}{l}{\(^{\dagger}\) denotes the proposed model.}\\
\end{tabular}%
}
\end{table}

\begin{table}[]
\centering
\caption{{Ablation study: classification accuracies on 5 selected binary classification datasets after adding label noise to 5\%, 10\%, 20\% of training samples.}}
\label{tab:ablation}
\resizebox{8cm}{!}{%
\begin{tabular}{|l|l|l|l|}
\hline
\textbf{Noise} & \textbf{Datasets}  & {\textbf{RKM} \cite{10.1162/neco_a_00984}} & {\textbf{CI-RKM}}$^{\dagger}$ \\ \hline
5\%   & conn\_bench\_sonar\_mines\_rocks & 88.1                    & 88.1                       \\
      & credit\_approval                 & 80.43                   & 82.61                      \\
      & heart\_hungarian                 & 72.88                   & 74.58                      \\
      & ilpd\_indian\_liver              & 64.1                    & 68.38                      \\
      & ionosphere                       & 91.55                   & 91.55                      \\
      & musk\_1                          & 87.5                    & 88.54                      \\ \hline

& Average Accuracy & 80.76	&82.29   \\   \hline
10\%  & conn\_bench\_sonar\_mines\_rocks & 76.19                   & 71.43                      \\
      & credit\_approval                 & 76.09                   & 76.09                      \\
      & heart\_hungarian                 & 74.58                   & 76.27                      \\
      & ilpd\_indian\_liver              & 64.96                   & 68.38                      \\
      & ionosphere                       & 78.87                   & 87.32                      \\
      & musk\_1                          & 78.13                   & 82.29                      \\  \hline

& Average Accuracy & 74.8 &	76.96  \\   \hline
20\%  & conn\_bench\_sonar\_mines\_rocks & 64.29                   & 61.9                       \\
      & credit\_approval                 & 69.57                   & 69.57                      \\
      & heart\_hungarian                 & 66.1                    & 61.02                      \\
      & ilpd\_indian\_liver              & 55.56                   & 60.68                      \\
      & ionosphere                       & 70.42                   & 76.06                      \\
      & musk\_1                          & 63.54                   & 70.83   \\             \hline

& Average Accuracy & 64.91&	66.68   \\ \hline
\multicolumn{4}{l}{\(^{\dagger}\) denotes the proposed model.}\\
\end{tabular}%
}
\end{table}



\subsection{Ablation Study}

In order to show effectiveness of the proposed CI-RKM and class-informed weight function, we performed an ablation study on $5$ selected benchmark datasets, namely conn\_bench\_sonar\_mines\_rocks, credit\_approval, heart\_hungarian, ilpd\_indian\_liver, ionosphere, and musk\_1, from different domain. 
Table \ref{tab:ablation} presents the performance of RKM and the proposed CI-RKM with label noise added to 5\%, 10\%, and 20\% of the training samples. As the label noise increases, CI-RKM consistently maintains a higher or comparable accuracy compared to RKM, showcasing its resilience to noisy data. Under moderate to high noise conditions, CI-RKM either maintains or improves its performance, indicating that the model is better equipped to generalize and remain stable despite data imperfections. This robustness is particularly important for real-world applications where noisy or incomplete data is common, highlighting the practical utility of the model in less-than-ideal conditions.

The superior performance of CI-RKM under noisy conditions suggests that the architecture and class-informed weight function incorporated into the model play a significant role in enhancing its ability to handle disturbances in the data. This makes CI-RKM a more reliable choice for tasks involving noisy or real-world datasets. While RKM also performs well under lower noise levels, CI-RKM consistently outperforms it as noise levels increase, demonstrating the effectiveness of the proposed model and the class-informed weight function in adapting to challenging conditions. The results in Table \ref{tab:ablation} support the claim that CI-RKM is not only the top performer in terms of accuracy but also the most resilient model when faced with noisy data, making it the most versatile and robust option in this comparison.

\section{Conclusions}\label{Conclusions-section}
This paper introduces the class-informed restricted kernel machine (CI-RKM), a novel extension of the restricted kernel machine designed to address the challenges posed by noisy data and outliers in classification tasks. By incorporating a class-informed weight function that adapts based on proximity to the class center, CI-RKM enhances the model’s robustness and classification accuracy. This innovation leverages geometric relationships between data points, improving the model’s ability to handle noisy datasets. The integration of class-specific information enhances the model’s adaptability, allowing it to capture subtle variations across different classes, leading to reliable and accurate predictions. Through extensive experimental evaluation, we demonstrate that the proposed CI-RKM outperforms the baseline models in terms of both performance and robustness, making it a valuable tool for a wide range of machine learning applications. 

In the future, to enhance the scalability of the RKM model, concepts from granular ball theory can be incorporated \cite{quadir2024granular}. Additionally, to enable the RKM to handle matrix input sample directly, the architecture can be developed by drawing inspiration from the support matrix machine \cite{kumari2024support}.

\bibliographystyle{IEEEtranN}
\bibliography{refs.bib}
\end{document}